\documentclass[12pt,dvips,generic,noinfoline]{imsart}

\RequirePackage[OT1]{fontenc}
\RequirePackage{amsthm,amsmath,amsfonts,amssymb,amsbsy,fancybox,natbib}
\RequirePackage[colorlinks,citecolor=blue,urlcolor=blue]{hyperref}
\RequirePackage{hypernat}
\usepackage[ruled,section]{algorithm}
\usepackage[noend]{algorithmic}
\usepackage{fullpage}
\usepackage{subfigure,epsfig,graphicx}
\usepackage{multirow,array,booktabs}

\startlocaldefs
\numberwithin{equation}{section}
\theoremstyle{plain}
\newtheorem{theorem}{Theorem}[section]

\newtheoremstyle{remark}{\topsep}{\topsep}%
     {\normalfont}
     {}           
     {\bfseries}  
     {.}          
     {.5em}       
     {\thmname{#1}\thmnumber{ #2}\thmnote{ #3}}
\theoremstyle{remark}

\endlocaldefs

\def\comma{\unskip,~}

\long\def\comment#1{}

\def\supp{\mathop{\text{supp}\kern.2ex}}
\def\argmin{\mathop{\text{\rm arg\,min}}}

\let\hat\widehat

\let\hat\widehat

\def\ds{\displaystyle}

\def\1{{(1)}}
\def\2{{(2)}}

\long\def\comment#1{}

\newenvironment{enum}{
\begin{enumerate}
  \setlength{\itemsep}{1pt}
  \setlength{\parskip}{0pt}
  \setlength{\parsep}{0pt}
}{\end{enumerate}}

\begin{document}

\begin{frontmatter}
\title{Stability Approach to Regularization Selection (StARS) for High 
Dimensional Graphical Models}
\runtitle{Graph-Valued Regression}

\begin{aug}
\author{\fnms{Han} \snm{Liu}\ead[label=e1]{hanliu@cs.cmu.edu}}
\comma
\author{\fnms{Kathryn} \snm{Roeder}\ead[label=e2]{lafferty@cs.cmu.edu}}
\and
\author{\fnms{Larry} \snm{Wasserman}\ead[label=e3]{larry@stat.cmu.edu}}

\address{Carnegie Mellon University\\[10pt]
}
     \end{aug}

\begin{abstract}
   A challenging problem in estimating high-dimensional graphical
  models is to choose the regularization parameter in a data-dependent
  way. The standard techniques include $K$-fold cross-validation
  ($K$-CV), Akaike information criterion (AIC), and Bayesian information
  criterion (BIC).  Though these methods work well for low-dimensional
  problems, they are not suitable in high dimensional settings.  In
  this paper, we present StARS: a new stability-based method for choosing the
  regularization parameter in high dimensional inference for
  undirected graphs.  The method has a clear interpretation: we use
  the least amount of regularization that simultaneously makes a graph
  sparse and replicable under random sampling.  This interpretation
  requires essentially no conditions.  Under mild conditions, we
  show that StARS is partially sparsistent in terms of graph
  estimation: i.e. with high probability, all the true edges will be included in the selected model even when the graph size diverges with the sample size.  Empirically, the
  performance of StARS is compared with the state-of-the-art
  model selection procedures, including $K$-CV, AIC, and BIC, on both
  synthetic data and a real microarray dataset.  StARS
  outperforms all these competing procedures.
\end{abstract}


\tableofcontents
\end{frontmatter}

\vspace{-15pt}
\section{Introduction}
\vspace{-10pt}

Undirected graphical models have emerged as a useful tool because they
allow for a stochastic description of complex associations in
high-dimensional data. For example, biological processes in a cell
lead to complex interactions among gene products. It is of interest to
determine which features of the system are conditionally independent.
Such problems require us to infer an undirected graph from
i.i.d. observations.  Each node in this graph corresponds to a random
variable and the existence of an edge between a pair of nodes
represent their conditional independence relationship.

Gaussian graphical models \citep{Dempster72, whit:1990,
  Edwa:1995, laur:1996} are by far the most popular approach for
learning high dimensional undirected graph structures.  Under the
Gaussian assumption, the graph can be estimated using the sparsity
pattern of the inverse covariance matrix.  If two variables are
conditionally independent, the corresponding element of the inverse
covariance matrix is zero.  In many applications, estimating the the inverse covariance matrix
is statistically challenging because the number of features measured
may be much larger than the number of collected samples. To handle this challenge, the {\em graphical lasso} or {\em glasso} \citep{FHT:07, Yuan:Lin:07, Banerjee:08} is rapidly becoming a
popular method for estimating sparse undirected graphs.  To use this method,
however, the user must specify a regularization parameter $\lambda$
that controls the sparsity of the graph. The choice of $\lambda$ is critical since different $\lambda$'s may lead to different scientific conclusions of the statistical inference. 
Other methods for estimating
high dimensional graphs include \citep{Meinshausen:2006,
Peng09,Liu_npn:09}. They also require the user to specify a regularization parameter.

The standard methods for choosing the regularization parameter are AIC \citep{akai:1973},
BIC \citep{Schw:1978} and cross validation \citep{Efro:1982}.  Though these methods have good theoretical properties in low dimensions, they are not suitable for high dimensional problems.  In
regression, cross-validation has been shown to overfit the data
\citep{Wasserman:09}.  Likewise, AIC and BIC tend to perform poorly
when the dimension is large relative to the sample size.  Our
simulations confirm that these methods perform poorly when used with
glasso.

A new approach to model selection, based on model stability, has
recently generated some interest in the literature \citep{Lange04}.  The idea, as we develop it, is based on subsampling \citep{Dimitris:99}
and builds on the approach of 
\cite{stability:10}.  We draw many random subsamples and construct a
graph from each subsample (unlike $K$-fold cross-validation, these
subsamples are overlapping).   We choose the regularization parameter so that the obtained graph is sparse and there is not
too much variability across subsamples.  More precisely, we start with
a large regularization which corresponds to an empty, and hence highly
stable, graph.  We gradually reduce the amount of regularization until there is a small
but acceptable amount of variability of the graph across subsamples.
In other words, we regularize to the point that we control the
dissonance between graphs. The procedure is named StARS: Stability
Approach to Regularization Selection. We study the performance of StARS by simulations and 
theoretical analysis in Sections 4 and 5.  Although we focus here on graphical models, StARS is quite general and
can be adapted to other settings including regression, classification, clustering, and dimensionality reduction.

In the context of clustering, results of stability methods have been
mixed.  Weaknesses of stability have been shown in
\citep{Ben-david06asober}.  However, the approach was successful for
density-based clustering \citep{Rinaldo:09}.  For graph selection, \cite{stability:10} also used a stability
criterion; however, their approach differs from StARS in its
fundamental conception.  They use subsampling to produce a new and more stable
regularization path then select a regularization parameter from this newly created path, whereas we propose to use subsampling to directly select
one regularization parameter from the original path.  Our aim is to ensure that
the selected graph is sparse, but inclusive, while they aim to
control the familywise type I errors.  As a consequence, their goal is contrary to ours: instead of selecting a larger graph that contains the true graph, they try to select a smaller graph that is contained in the true graph.  As we will discuss in Section 3, in specific application domains like gene regulatory network analysis,  our goal for graph selection is more natural.

In Section \ref{sec::gr} we review the basic notion of estimating high
dimensional undirected graphs; in Section \ref{sec::method} we develop StARS; in
Section \ref{sec::theory} we present a theoretical analysis of the proposed method;
and in Section \ref{sec::experiment} we report experimental results on both simulated data
and a gene microarray dataset, where the problem is to construct gene
regulatory network based on natural variation of the expression levels
of human genes.

\section{Estimating a High-dimensional Undirected Graph}
\label{sec::gr}

Let $X=\bigl(X(1),\ldots, X(p)\bigr)^T$
be a random vector with distribution $P$.
The undirected graph
$G=(V,E)$ associated with $P$
has vertices $V=\{X(1), \ldots, X(p)\}$
and a set of edges $E$ corresponding to pairs of vertices. In this paper, we also interchangeably use $E$ to denote the adjacency matrix of the graph $G$. 
The edge corresponding to $X(j)$ and $X(k)$ is absent
if $X(j)$ and $X(k)$ are conditionally independent given the other
coordinates of $X$. The graph estimation problem is to infer $E$ from i.i.d. observed data
$X_1,\ldots, X_n$ where
$X_i = (X_i(1),\ldots, X_i(p))^T$.

Suppose now that $P$ is Gaussian with mean vector $\mu$ and covariance
matrix $\Sigma$.  Then the edge corresponding to $X(j)$ and $X(k)$ is
absent if and only if $\Omega_{jk}=0$ where $\Omega = \Sigma^{-1}$.
Hence, to estimate the graph we only need to estimate the sparsity pattern of $\Omega$.  When
$p$ could diverge with $n$, estimating $\Omega$ is difficult.  A popular approach is
the {\em graphical lasso} or {\em glasso} \citep{FHT:07, Yuan:Lin:07,
Banerjee:08}.  Using glasso, we estimate $\Omega$ as follows:  Ignoring constants, the log-likelihood (after maximizing over $\mu$) can be written as
$$\ell(\Omega) = \log |\Omega| - \text{trace}\bigl(\hat{\Sigma}\Omega \bigr)$$
where $\hat \Sigma$
is the sample covariance matrix.  With a positive regularization parameter $\lambda$, the glasso estimator $\hat\Omega(\lambda)$  is
obtained by minimizing the regularized negative log-likelihood
\begin{equation}
\hat{\Omega}(\lambda)= \argmin_{\Omega\succ 0}\Bigl\{-\ell(\Omega) + \lambda ||\Omega||_1\Bigr\} \label{eq::glassoformula}
\end{equation}
where $||\Omega||_1 = \sum_{j,k} |\Omega_{jk}|$ is the elementwise $\ell_{1}$-norm of $\Omega$.  The estimated graph $\hat G(\lambda) = (V, \hat{E}(\lambda))$ is
then easily obtained from $\hat\Omega (\lambda)$: for $i\neq j$, an edge $(i,j) \in \hat{E}(\lambda)$  if and only if the
corresponding entry in $\hat\Omega (\lambda)$ is nonzero. \cite{FHT:07} give a fast algorithm for calculating
$\hat\Omega(\lambda)$ over a grid of $\lambda$s ranging from small to large.  By taking advantage of the fact that
the objective function in \eqref{eq::glassoformula} is convex, their algorithm iteratively estimates a single row (and column) of $\Omega$ in each iteration by solving a lasso regression \citep{Tibshirani:96}.
The resulting regularization path $\hat\Omega(\lambda)$ for all $\lambda$s  has been shown to have excellent theoretical properties
\citep{Rothman:08, Ravikumar:Gauss:09}. For example, \cite{Ravikumar:Gauss:09} show that, if the regularization parameter $\lambda$ satisfies a certain rate, the corresponding estimator $\hat{\Omega}(\lambda)$ could recover the true graph with high probability.  However, these types of results are either asymptotic or non-asymptotic but with very large constants. They are not practical enough to guide the choice of the regularization parameter $\lambda$ in finite-sample settings.

\section{Regularization Selection}
\label{sec::method}

In Equation \eqref{eq::glassoformula}, the choice of $\lambda$ is critical because $\lambda$ controls the
sparsity level of $\hat G(\lambda)$. Larger values of $\lambda$ tend to
yield sparser graphs and smaller values of $\lambda$ yield denser graphs.   It is convenient to define $\Lambda =
1/\lambda$ so that small $\Lambda$ corresponds to a more sparse graph.
In particular, $\Lambda=0$ corresponds to the empty graph with no
edges. Given a grid of regularization parameters $\mathcal{G}_{n} = \{\Lambda_{1}, \ldots, \Lambda_{K}\}$,  our goal of graph regularization parameter selection is to choose one $\hat{\Lambda} \in \mathcal{G}_{n}$, such that the true graph $E$ is contained in $\hat{E}(\hat{\Lambda})$ with high probability. In other words, we want to ``overselect'' instead of ``underselect''. Such a choice is motivated by application problems like gene regulatory networks reconstruction,  in which we aim to study the interactions of many genes.  For these types of studies, we tolerant some false positives but not false negatives. Specifically, it is acceptable that  an edge presents but the two genes corresponding to this edge do not really interact with each other. Such false positives can generally be screened out by  more fine-tuned downstream biological experiments. However, if one important interaction edge  is omitted at the beginning, it's very difficult for us to re-discovery it by follow-up analysis. There is also a tradeoff: we want to select a denser graph which contains the true graph with high probability. At the same time, we want the graph to be as sparse as possible so that  important information will not be buried by massive false positives.  Based on this rationale, an ``underselect'' method, like the approach of \cite{stability:10}, does not really fit our goal. In the following, we start with an overview of several state-of-the-art regularization parameter selection methods for graphs. We then  introduce our new StARS approach.

\subsection{Existing Methods}

The regularization parameter is often chosen using AIC or BIC.  Let
$\hat\Omega(\Lambda)$ denote the estimator corresponding to $\Lambda$. Let $d(\Lambda)$ denote the degree of freedom (or the effective number of free parameters) of the corresponding Gaussian model.   AIC chooses $\Lambda$ as
\begin{equation}
\mathrm{(AIC)}\quad \hat{\Lambda} = \argmin_{\Lambda\in\mathcal{G}_{n}}\bigl\{-2\ell\bigl(\hat\Omega(\Lambda)\bigr) + 2d(\Lambda)\bigr\},
\end{equation}
and
BIC chooses $\Lambda$ as
\begin{equation}
\mathrm{(BIC)}\quad \hat{\Lambda} = \argmin_{\Lambda\in\mathcal{G}_{n}}\bigl\{-2\ell\bigl(\hat\Omega(\Lambda)\bigr) + d(\Lambda)\cdot\log n\bigr\}.
\end{equation}
The usual theoretical justification for these methods assumes that the
dimension $p$ is fixed as $n$ increases; however, in the case where $p >
n$ this justification is not applicable.  In fact, it's even not straightforward how to estimate the degree of freedom $d(\Lambda)$ when $p$ is larger than $n$ . A common practice is to calculate $d(\Lambda)$ as $d(\Lambda) = m(\Lambda)(m(\Lambda)-1)/2
+ p$ where $m(\Lambda)$ denotes the number of nonzero elements of
$\hat\Omega(\Lambda)$. As we will see in our
experiments, AIC and BIC tend to select overly dense graphs in high dimensions.

Another popular method is $K$-fold cross-validation ($K$-CV).  For this procedure the data is partitioned into $K$ subsets. Of the $K$ subsets one is retained as the validation data, and the remaining $K -1$ ones are used as training data. For each $\Lambda \in \mathcal{G}_{n}$, we estimate a graph on the $K-1$ training sets and evaluate the negative log-likelihood on the retained validation set. The results are averaged over all $K$ folds to obtain a single CV score.  We then choose
$\Lambda$ to minimize the CV score over he whole grid $\mathcal{G}_{n}$.  In regression,
cross-validation has been shown to overfit \citep{Wasserman:09}.  Our
experiments will confirm this is true for graph estimation as well.

\subsection{StARS: Stability Approach to Regularization Selection}

The StARS approach is to choose $\Lambda$ based on stability.
When $\Lambda$
is 0, the graph is empty and two
datasets from $P$ would both yield the same graph.  As we increase
$\Lambda$, the variability of the graph increases and hence the
stability decreases.  We increase $\Lambda$ just until the point where
the graph becomes variable as measured by the stability. StARS
leads to a concrete rule for choosing $\Lambda$.

Let $b = b(n)$ be such that $1 < b(n) < n$.  We draw $N$ random
subsamples $S_1,\ldots, S_N$ from $X_1,\ldots, X_n$, each of size $b$.  There are
$\binom{n}{b}$ such subsamples.  Theoretically one uses all
$\binom{n}{b}$ subsamples.  However,  \cite{Dimitris:99} show that it
suffices in practice to choose a large number $N$ of subsamples at
random.  Note that, unlike bootstrapping \citep{Efro:1982}, each subsample is drawn
without replacement.  For each $\Lambda\in\mathcal{G}_{n}$, we construct a graph using the glasso
for each subsample.  This results in $N$ estimated edge matrices $\hat{E}^{b}_1(\Lambda),\ldots, \hat{E}^{b}_N(\Lambda)$. Focus for now on one edge $(s,t)$ and one value of $\Lambda$.  Let $\psi^{\Lambda}(\cdot)$ denote the glasso algorithm with the regularization parameter $\Lambda$. 
For any subsample $S_{j}$ let $\psi_{st}^{\Lambda}(S_{j})=1$ if the algorithm puts an edge and
$\psi^{\Lambda}_{st}(S_{j})=0$ if the algorithm does not put an edge between $(s,t)$.
Define $$\theta^{b}_{st}(\Lambda) = \mathbb{P}(\psi^{\Lambda}_{st}(X_1,\ldots, X_b)=1).$$
To estimate $\theta^{b}_{st}(\Lambda)$, we use a U-statistic of order $b$, namely,
$$\ds  \hat{\theta}^{b}_{st}(\Lambda) = \frac{1}{N}\sum_{j=1}^N \psi^{\Lambda}_{st}(S_j).$$

Now define the parameter $\xi^{b}_{st}(\Lambda) = 2\theta^{b}_{st}(\Lambda)(1-\theta^{b}_{st}(\Lambda))$ and let $\hat\xi^{b}_{st}(\Lambda) =
2\hat\theta^{b}_{st}(\Lambda)(1-\hat\theta^{b}_{st}(\Lambda))$ be its estimate.  Then $\xi^{b}_{st}(\Lambda)$, in addition to being
twice the variance of the Bernoulli indicator of the edge $(s, t)$, has the
following nice interpretation:  For each pair of graphs, we can ask
how often they disagree on the presence of the edge: $\xi^{b}_{st}(\Lambda)$ is the
fraction of times they disagree.  For $\Lambda \in \mathcal{G}_{n}$, we regard $\xi^{b}_{st}(\Lambda)$ as a measure of
instability of the edge across subsamples, with $0\leq \xi^{b}_{st}(\Lambda)\leq 1/2$.

Define the total instability by averaging over all edges:
$$\hat{D}_{b}(\Lambda) = 
\frac{ \sum_{s<t}\hat\xi^{b}_{st}}{ \binom{p}{2}}.$$
Clearly on the boundary $\hat{D}_{b}(0)=0$, and $\hat{D}_{b}(\Lambda)$ generally will
increase as $\Lambda$ increases.  However, when $\Lambda$ gets very
large, all the graphs will become dense and $\hat{D}_{b}(\Lambda)$ will begin to
decrease.  Subsample stability for large $\Lambda$ is essentially an
artifact.  We are interested in stability for sparse graphs not dense
graphs.  For this reason we monotonize $\hat{D}_{b}(\Lambda)$ by defining $\overline{D}_{b}(\Lambda) = \sup_{0 \leq t \leq \Lambda} \hat{D}_{b}(t)$. Finally, our StARS approach chooses $\Lambda$ by defining
$\hat\Lambda_{\rm s} = \sup \Bigl\{ \Lambda :\ \overline{D}_{b}(\Lambda) \leq \beta \Bigr\}$
for a specified cut point value $\beta$.

It may seem that we have merely replaced the problem of choosing
$\Lambda$ with the problem of choosing $\beta$, but $\beta$ is an
interpretable quantity and we always set a default value
$\beta = 0.05$.  One thing to note is that all quantities $\hat{E}, \hat{\theta}, \hat{\xi}, \hat{D}$ depend on the subsampling block size $b$. Since StARS is based on subsampling, the effective sample size for estimating the selected graph is $b$ instead of $n$. Compared with methods like BIC and AIC which fully utilize all $n$ data points. StARS has some efficiency loss in low dimensions. However, in high dimensional settings,  the gain of StARS on better graph selection significantly dominate this efficiency loss. This fact is confirmed by our experiments.

\section{Theoretical Properties}
\label{sec::theory}

The StARS procedure is quite general  and can be applied with any graph estimation algorithms. Here, we provide its theoretical properties. We start with a key theorem which establishes the rates of convergence of the estimated stability quantities to their population means. We then discuss the implication of this theorem on general gaph regularization selection problems.

Let $\Lambda$ be an element in the grid ${\cal G}_n=\{\Lambda_1,\ldots, \Lambda_K\}$ where $K$ is a polynomial of $n$. We denote  $D_{b}(\Lambda)=\mathbb{E}(\hat{D}_{b}(\Lambda))$. The quantity $\hat\xi^{b}_{st}(\Lambda)$ is an estimate
of $\xi^{b}_{st}(\Lambda)$ and
$\hat{D}_{b}(\Lambda)$ is an estimate of 
${D}_{b}(\Lambda)$.  Standard $U$-statistic theory guarantees that these estimates
have good uniform convergence properties to their population quantities:

\begin{theorem}\label{thm::key}
{\rm (Uniform Concentration)} The following statements hold with no assumptions on $P$. For any $\delta \in (0,1)$, with probability at least $1-\delta$, we have
\begin{eqnarray}
\forall \Lambda \in \mathcal{G}_{n},~\max_{s<t} | \hat\xi^{b}_{st}(\Lambda) - \xi^{b}_{st}(\Lambda)| \leq \sqrt{\frac{18b\left(2\log p + \log (2/\delta)\right)}{n}},  \label{eq::concentration1}\\
\max_{\Lambda\in{\cal G}_n} |\hat{D}_{b}(\Lambda) - D_{b}(\Lambda)| \leq \sqrt{\frac{18b\left( \log K + 4 \log p + \log\left( 1/\delta\right) \right)}{{n}}}. \label{eq::concentration2}
\end{eqnarray}
\end{theorem}

\begin{proof}
Note that
$\hat \theta^{b}_{st}(\Lambda)$ is a $U$-statistic of order $b$.
Hence, by Hoeffding's inequality for $U$-statistics \citep{serf:1980}, we have, for any $\epsilon > 0$, 
\begin{equation}\label{eq::U}
\mathbb{P}(|\hat\theta^{b}_{st}(\Lambda) -\theta^{b}_{st}(\Lambda)| > \epsilon) \leq 2 \exp\left({-2n\epsilon^2/b}\right).
\end{equation}
Now
$\hat\xi^{b}_{st}(\Lambda)$ is just a function of the $U$-statistic $\hat\theta^{b}_{st}(\Lambda)$. Note that 
\begin{eqnarray}
|\hat\xi^{b}_{st}(\Lambda) - \xi^{b}_{st}(\Lambda)| & = &
2|\hat\theta^{b}_{st}(\Lambda)(1-\hat\theta^{b}_{st}(\Lambda)) - \theta^{b}_{st}(\Lambda)(1-\theta^{b}_{st}(\Lambda))| \\
& = &
2|\hat\theta^{b}_{st}(\Lambda) -\bigl(\hat\theta^b_{st}(\Lambda) \bigr)^{2}  - \theta^{b}_{st}(\Lambda)+\bigl(\theta^b_{st}(\Lambda)\bigr)^{2}|\\
& \leq &
2|\hat\theta^{b}_{st}(\Lambda) - \theta^{b}_{st}(\Lambda)| + 2|\bigl(\hat\theta^b_{st}(\Lambda) \bigr)^{2} - \bigl(\theta^b_{st}(\Lambda)\bigr)^{2}|\\
&\leq &
2|\hat\theta^{b}_{st}(\Lambda) - \theta^{b}_{st}(\Lambda)| + 2|(\hat\theta^{b}_{st}(\Lambda)  - \theta^{b}_{st}(\Lambda))(\hat\theta^{b}_{st}(\Lambda)  + \theta^{b}_{st}(\Lambda))|\\
& \leq &
2|\hat\theta^{b}_{st}(\Lambda) - \theta^{b}_{st}(\Lambda)| + 4|\hat\theta^{b}_{st}(\Lambda)  - \theta^{b}_{st}(\Lambda)|\\
&=&
6|\hat\theta^{b}_{st}(\Lambda) - \theta^{b}_{st}(\Lambda)|,
\end{eqnarray}
we have
$|\hat\xi^{b}_{st}(\Lambda) - \xi^{b}_{st}(\Lambda)| \leq 6 |\hat\theta^{b}_{st}(\Lambda) - \theta^{b}_{st}(\Lambda)|$.
Using (\ref{eq::U})
and the union bound over all the edges, we obtain:
for each $\Lambda \in \mathcal{G}_{n}$,
\begin{equation}\label{eq::ineq}
\mathbb{P}(\max_{s<t} | \hat\xi^{b}_{st}(\Lambda) - \xi^{b}_{st}(\Lambda)| > 6\epsilon)\leq
2 p^2 \exp\left({-2n\epsilon^2/b}\right).
\end{equation}
Using two union bound arguments
over the $K$ values of $\Lambda$ and all the $p(p-1)/2$ edges, we have:
\begin{eqnarray}
\mathbb{P}\left(\max_{\Lambda\in{\cal G}_n} |\hat{D}_{b}(\Lambda) - D_{b}(\Lambda)|  \geq  \epsilon \right) &\leq &|\mathcal{G}_{n}|\cdot \frac{p(p-1)}{2}\cdot \mathbb{P}(\max_{s<t} | \hat\xi^{b}_{st}(\Lambda) - \xi^{b}_{st}(\Lambda)| > \epsilon) \\
& \leq & K\cdot p^{4}\cdot \exp\left(-{n\epsilon^2}/{(18b)}\right). 
\end{eqnarray}
Equations \eqref{eq::concentration1} and  \eqref{eq::concentration2}  follow directly from \eqref{eq::ineq} and the above exponential probability inequality.
\end{proof}

Theorem \ref{thm::key} allows us to explicitly characterize the high-dimensional scaling of the sample size $n$, dimensionality $p$, subsampling block size $b$, and the grid size $K$. More specifically,  we get
\begin{equation}
\frac{n}{b\log\bigl(np^{4}K \bigr)} \rightarrow \infty \implies \max_{\Lambda\in{\cal G}_n} |\hat{D}_{b}(\Lambda) - D_{b}(\Lambda)| \stackrel{P}{\to} 0 \label{eq::scaling}
\end{equation}
by setting $\delta = 1/n$  in Equation \eqref{eq::concentration2}.  From \eqref{eq::scaling}, let $c_{1}, c_{2}$  be arbitrary positive constants, if $b = c_{1}\sqrt{n}$, $K= n^{c_{2}}$, and $p \leq \exp\left({n^{\gamma}}\right)$ for some $\gamma <1/2$,  the estimated total stability $\hat{D}_{b}(\Lambda)$ still converges to its mean $D_{b}(\Lambda)$ uniformly over the whole grid $\mathcal{G}_{n}$. 

We now discuss the implication of Theorem \ref{thm::key} to graph regularization selection problems. Due to the generality of StARS, we provide theoretical justifications for a whole family of graph estimation procedures satisfying certain conditions. Let $\psi$ be a graph estimation procedure. We denote $\hat{E}^{b}(\Lambda)$ as the estimated edge set using the regularization parameter $\Lambda$ by applying $\psi$ on a subsampled dataset with block size $b$.   To establish graph selection result,  we start with two technical assumptions:
\begin{itemize}
\item[(A1)] $\exists \Lambda_{\rm o} \in  \mathcal{G}_{n}$, such that  $\max_{\Lambda \leq \Lambda_{\rm o} \wedge \Lambda \in \mathcal{G}_{n}} D_{b}(\Lambda)  \leq {\beta}/{2} $ for large enough $n$.
\item[(A2)] For any $\Lambda \in \mathcal{G}_{n}$ and $\Lambda \geq \Lambda_{\rm o}$, $\mathbb{P}\bigl( E \subset \hat{E}^{b}(\Lambda)  \bigr) \rightarrow 1$ as $n\rightarrow \infty$. 
\end{itemize}
Note that $\Lambda_{\rm o}$ here depends on the sample size $n$ and does not have to be unique.  To understand the above conditions, (A1) assumes that there exists a threshold $\Lambda_{\rm o}\in\mathcal{G}_{n}$, such that the population quantity $ D_{b}(\Lambda)$ is small for all $\Lambda \leq \Lambda_{\rm o}$. (A2) requires that all estimated graphs using regularization parameters $\Lambda \geq \Lambda_{\rm o}$ contain the true graph with high probability. Both assumptions are mild and should be satisfied by most graph estimation algorithm with reasonable behaviors.  There is a tradeoff on the design of the subsampling block size $b$ . To make  (A2) hold, we require $b$ to be large. However, to make $\hat{D}_{b}(\Lambda)$ concentrate to  $D_{b}(\Lambda)$ fast, we require $b$ to be small. Our suggested value is $b = \lfloor10\sqrt{n}\rfloor$, which balances both the theoretical and empirical performance well. The next theorem provides the graph selection performance of StARS:

\begin{theorem}\label{thm::sparsistency}
{\rm (Partial Sparsistency):}   Let $\psi$ to be a graph estimation algorithm. We assume (A1) and (A2) hold for $\psi$ using $b=\lfloor10\sqrt{n}\rfloor$ and  $|\mathcal{G}_{n}| = K = n^{c_{1}}$ for some constant $c_{1} >0$. Let $\hat{\Lambda}_{\rm s} \in \mathcal{G}_{n}$ be the selected regularization parameter using the StARS procedure with a constant cutting point $\beta$. Then, if $p \leq \exp\left({n^{\gamma}}\right)$ for some $\gamma <1/2$, we have 
\begin{equation}
 \mathbb{P}\bigl( E \subset \hat{E}^{b}(\hat{\Lambda}_{\rm s}) \bigr)  \rightarrow 1 ~~\text{as}~~n\rightarrow \infty.
\end{equation}
\end{theorem}
\begin{proof}
We define $\mathcal{A}_{n}$ to be  the event that $\max_{\Lambda\in{\cal G}_n} |\hat{D}_{b}(\Lambda) - D_{b}(\Lambda)| \leq \beta/2$.  The scaling of $n, K, b, p$ in the theorem satisfies the L.H.S. of \eqref{eq::scaling}, which implies that $\mathbb{P}(\mathcal{A}_{n})\rightarrow 1$ as $n\rightarrow \infty$.

Using  (A1), we know that, on $\mathcal{A}_{n}$,
\begin{equation}
\max_{\Lambda \leq \Lambda_{\rm o} \wedge \Lambda \in \mathcal{G}_{n}} \hat{D}_{b}(\Lambda)  \leq \max_{\Lambda\in{\cal G}_n} |\hat{D}_{b}(\Lambda) - D_{b}(\Lambda)|  + \max_{\Lambda \leq \Lambda_{\rm o} \wedge \Lambda \in \mathcal{G}_{n}} D_{b}(\Lambda)  \leq \beta.
\end{equation}
This implies that, on $\mathcal{A}_{n}$, $\hat{\Lambda}_{\rm s} \geq {\Lambda}_{\rm o}$. The result  follows by applying (A2) and a union bound.
\end{proof}

\section{Experimental Results}
\label{sec::experiment}

We now provide empirical evidence to illustrate the usefulness of
StARS and compare it with several state-of-the-art competitors, including
$10$-fold cross-validation ($K$-CV), BIC, and AIC. For StARS we always use subsampling block size $b(n) = \lfloor10\cdot\sqrt{n}]$ and set the cut point $\beta=0.05$.  We first quantitatively
evaluate these methods on two types of synthetic datasets, where the true graphs are known.  We then illustrate StARS on a
microarray dataset that records the gene expression levels from
immortalized B cells of human subjects.  On all high dimensional
synthetic datasets, StARS significantly outperforms its
competitors. On the microarray dataset,
StARS obtains a remarkably simple graph while all competing
methods select what appear to be overly dense graphs.

\subsection{Synthetic Data}

To quantitatively evaluate the graph estimation performance, we
adapt the criteria including precision, recall, and $F_{1}$-score from
the information retrieval literature.  Let $G=(V,E)$ be a
$p$-dimensional graph and let $\hat{G} = (V, \hat{E})$ be an estimated
graph.  We define 
\begin{eqnarray}
\text{precision} = \frac{|\hat{E} \cap
E|}{|\hat{E}|},~~\text{recall} = \frac{|\hat{E} \cap
E|}{|E|},~~F_1\text{-score} = 2\cdot\frac{\text{precision}\cdot
\text{recall}}{\text{precision}+ \text{recall}}.
\end{eqnarray}
In other words, 
Precision is the number of correctly estimated edges divided by the
total number of edges in the estimated graph; recall is the number of correctly
estimated edges divided by the total number of edges in the true
graph; the $F_{1}$-score can be viewed as a weighted average of the
precision and recall, where an $F_{1}$-score reaches its best value at
1 and worst score at 0. On the synthetic data where we know the true
graphs, we also compare the previous methods with an oracle
procedure which selects the optimal regularization parameter by minimizing the total number of different edges between the estimated and true graphs along the full regularization path. Since this oracle procedure requires the
knowledge of the truth graph, it is not a practical method.
We only present it here to calibrate the inherent challenge of each simulated
scenario. To make the comparison fair, once the regularization parameters are selected, we estimate the oracle and StARS graphs only based on a subsampled dataset with size $$b(n) = \lfloor10\sqrt{n}\rfloor.$$ In contrast, the $K$-CV, BIC, and AIC graphs are estimated using the full dataset. More details about this issue were discussed in Section \ref{sec::method}.

We generate data from sparse Gaussian graphs, {\it neighborhood graphs} and {\it hub graphs},  which mimic
characteristics of real-wolrd biological networks. The mean is set to be zero and the covariance matrix
$\Sigma=\Omega^{-1}$. For both graphs, the diagonal elements of
$\Omega$ are set to be one. More specifically:
\begin{enum}
\item{\it Neighborhood graph}: We first uniformly sample $y_{1},
  \ldots, y_{n} $ from a unit square. We then set $\Omega_{ij} = \Omega_{ji} =
  \rho$ with probability $\left(\sqrt{2\pi}\right)^{-1} \exp\left(
    -4\|y_i - y_j \|^2 \right)$. All the rest $\Omega_{ij}$ are set to be
  zero. The number of nonzero off-diagonal elements of each row or column is
  restricted to be smaller than $\lfloor1/\rho \rfloor$. In this paper, $\rho$ is set to be 0.245.
\item{\it Hub graph}: 
The rows/columns are partitioned into  $J$ equally-sized disjoint groups: 
$V_{1}\cup V_{2}\ldots \cup V_{J} = \{1,\ldots, p\}$,  each group is 
associated with a ``pivotal'' row $k$.   Let $|V_{1}| = s$. 
We set $\Omega_{ik} =\Omega_{ki} = \rho$ 
for $i \in V_{k}$ and $\Omega_{ik}=\Omega_{ki}=0$ otherwise. In our experiment, $J = \lfloor p/s \rfloor$, $k = 1, s+1, 2s+1, \ldots$, and we always set $\rho= 1/(s+1)$ with $s=20$. 
\end{enum}
We generate synthetic datasets in both low-dimensional ($n=800, p
=40$) and high-dimensional ($n=400, p =100$) settings. Table \ref{tab:neighborhood} provides comparisons of all
methods, where we
repeat the experiments 100 times and report the averaged precision,
recall, $F_{1}$-score with their standard errors. 

\newcolumntype {Q}{>{$\displaystyle}l<{$}}
\newcolumntype {A}{>{$}c <{$}}
\begin{table}[h]
 \caption{\small Quantitative comparison of different methods on the datasets from the neighborhood and hub graphs.}
{\sf \footnotesize
\begin{tabular}{QAAA|QAA}
\toprule %
\multicolumn{6}{c}{Neighborhood graph:  n =800, p=40} {Neighborhood graph: n=400, p =100} \\
\cmidrule(r){2-4}\cmidrule(r){5-7}
   \mathbf{Methods}       & \text{Precision}  & \text{Recall} & F_{1}\text{-score} & \text{Precision}   &  \text{Recall}  & F_{1}\text{-score}\\
\midrule
 \text{Oracle}  & 0.9222 \ (0.05) &  0.9070 \ (0.07) & 0.9119 \ (0.04) &   0.7473 \ (0.09) & 0.8001\ (0.06) &0.7672\ (0.07) \\
\text{StARS} & 0.7204 \ (0.08) & 0.9530 \ (0.05) & 0.8171 \ (0.05) &   0.6366 \ (0.07) & 0.8718\ (0.06) & 0.7352 \ (0.07)  \\
\text{K-CV}     & 0.1394 \ (0.02) & 1.0000 \ (0.00) & 0.2440 \ (0.04) & 0.1383 \ (0.01) & 1.0000\ (0.00) & 0.2428\ (0.01) \\
\text{BIC}       & 0.9738 \ (0.03) & 0.9948 \ (0.02) & 0.9839 \ (0.01) &  0.1796 \ (0.11) & 1.0000\ (0.00) & 0.2933\ (0.13) \\
\text{AIC}       & 0.8696 \ (0.11) & 0.9996\ (0.01) & 0.9236 \ (0.07) &  0.1279 \ (0.00) & 1.0000\ (0.00) & 0.2268\ (0.01) \\
\hline
\end{tabular}
\begin{tabular}{QAAA|QAA}
\hline %
\multicolumn{6}{c}{Hub graph:  n =800, p=40} {Hub graph: n=400, p =100} \\
\cmidrule(r){2-4}\cmidrule(r){5-7}
   \mathbf{Methods}       & \text{Precision}  & \text{Recall} & F_{1}\text{-score} & \text{Precision}   &  \text{Recall}  & F_{1}\text{-score}\\
\midrule
 \text{Oracle}  & 0.9793 \ (0.01) &  1.0000 \ (0.00) & 0.9895 \ (0.01) &    0.8976 \ (0.02) & 1.0000\ (0.00) &0.9459\ (0.01) \\
\text{StARS} & 0.4377 \ (0.02) & 1.0000 \ (0.00) & 0.6086 \ (0.02) &   0.4572 \ (0.01) & 1.0000\ (0.00) & 0.6274 \ (0.01)  \\
\text{K-CV}     &0.2383 \ (0.09) & 1.0000 \ (0.00) & 0.3769 \ (0.01) & 0.1574 \ (0.01) & 1.0000\ (0.00) & 0.2719\ (0.00) \\
\text{BIC}       & 0.4879\ (0.05) & 1.0000 \ (0.00) & 0.6542 \ (0.05) &  0.2155 \ (0.00) & 1.0000\ (0.00) & 0.3545\ (0.01) \\
\text{AIC}       & 0.2522 \ (0.09) & 1.0000 \ (0.00) & 0.3951 \ (0.00) &  0.1676 \ (0.00) & 1.0000\ (0.00) & 0.2871\ (0.00) \\
\bottomrule
\end{tabular} }
  \label{tab:neighborhood}
\end{table}

For low-dimensional settings where $n \gg p$, the BIC criterion is very
competitive and performs the best among all the methods.  In high dimensional settings, however, StARS clearly outperforms all
the competing methods for both neighborhood and hub graphs. This is
consistent with our theory.  At first sight, it might be surprising
that for data from low-dimensional neighborhood graphs, BIC and AIC
even outperform the oracle procedure!  This is due to the fact that
both BIC and AIC graphs are estimated using all the $n=800$ data
points, while the oracle graph is estimated using only the subsampled dataset with size $b(n) = \lfloor10\cdot\sqrt{n} \rfloor= 282$. Direct usage
of the full sample is an advantage of model selection methods that
take the general form of BIC and AIC.  In high dimensions, however, we
see that even with this advantage, StARS clearly outperforms BIC and
AIC. The estimated graphs for different methods in the setting $n=400,
p =100$ are provided in Figures \ref{fig:neighborhood} and
\ref{fig:hub}, from which we see that the StARS graph is almost as
good as the oracle, while the $K$-CV, BIC, and AIC graphs are overly
too dense.


\begin{figure}[ht!]
\centering
  \begin{small}
  \begin{tabular}{ccc}
  \includegraphics[scale=0.36]{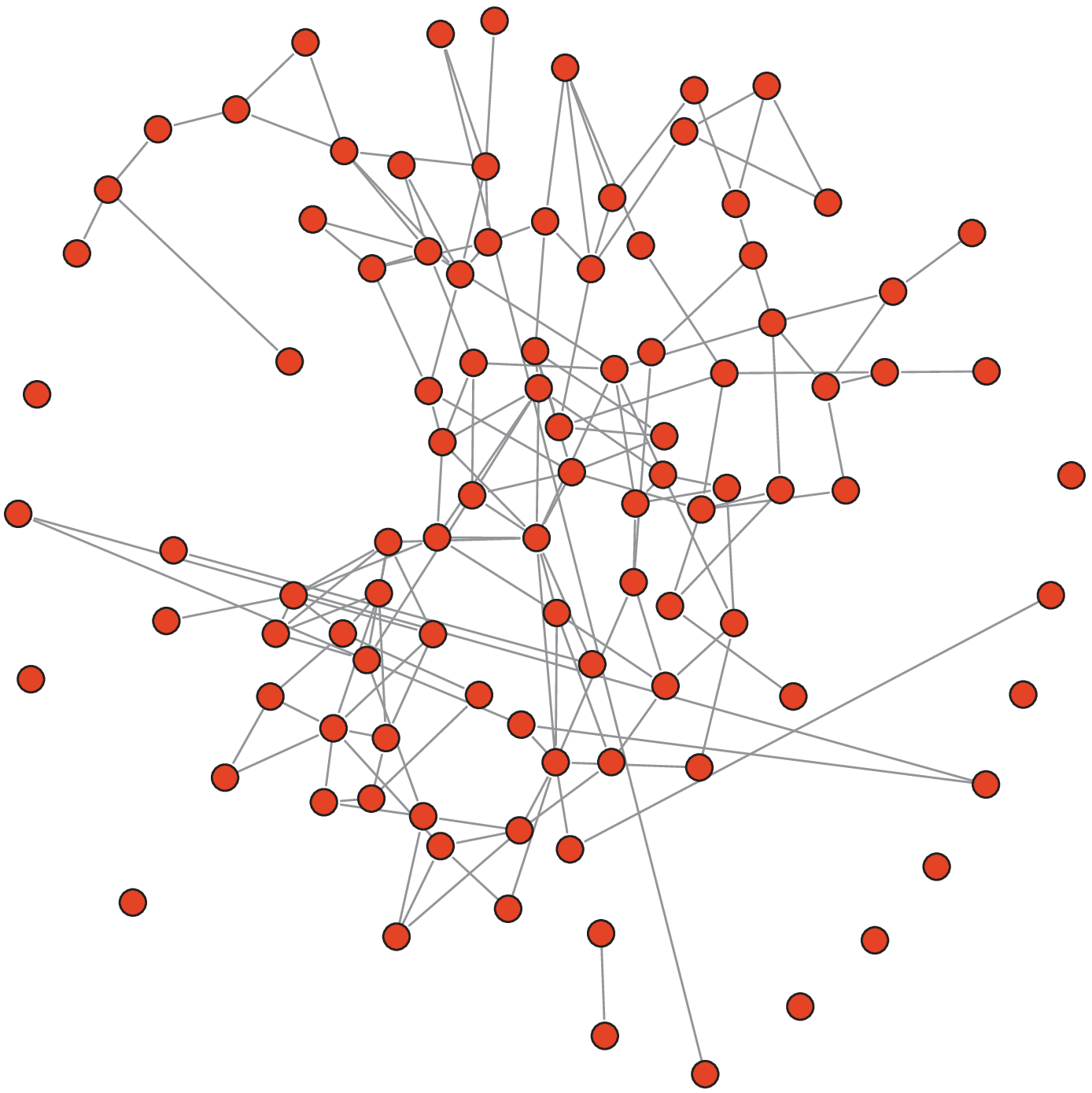} &
  \includegraphics[scale=0.36]{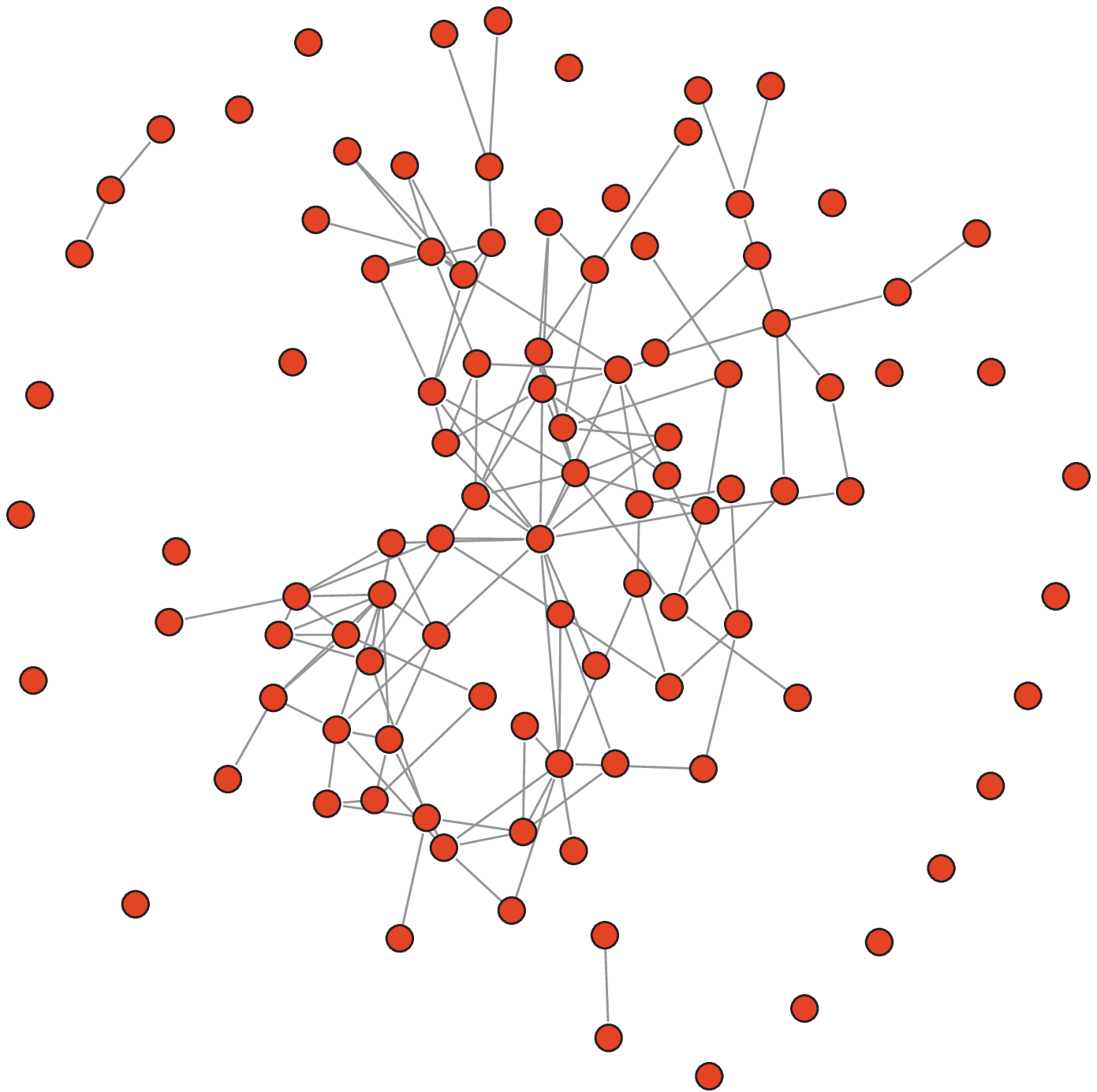} &
  \includegraphics[scale=0.36]{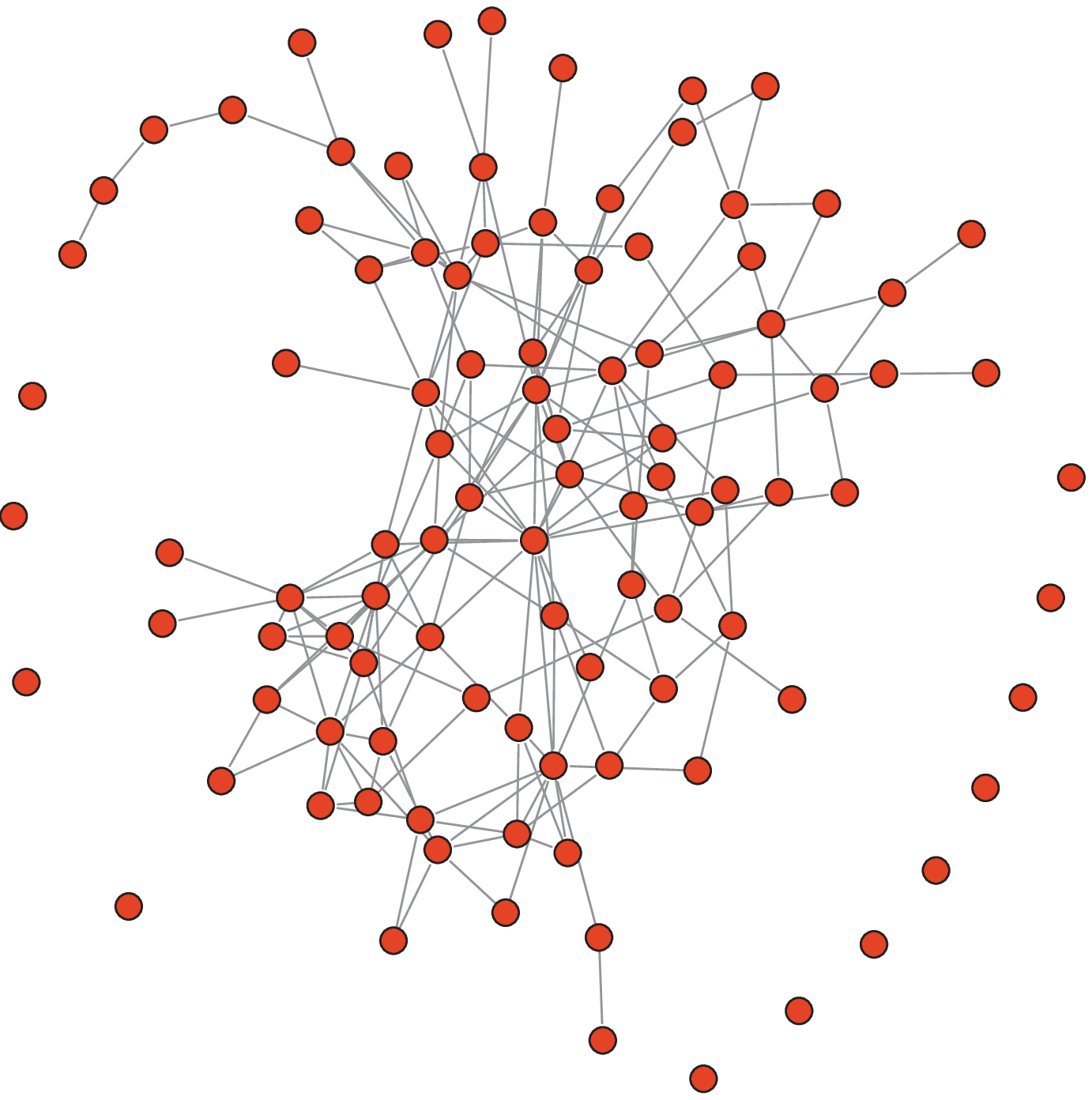}\\
     (a) True graph & (b) Oracle graph & (c) StARS graph\\
  \includegraphics[scale=0.36]{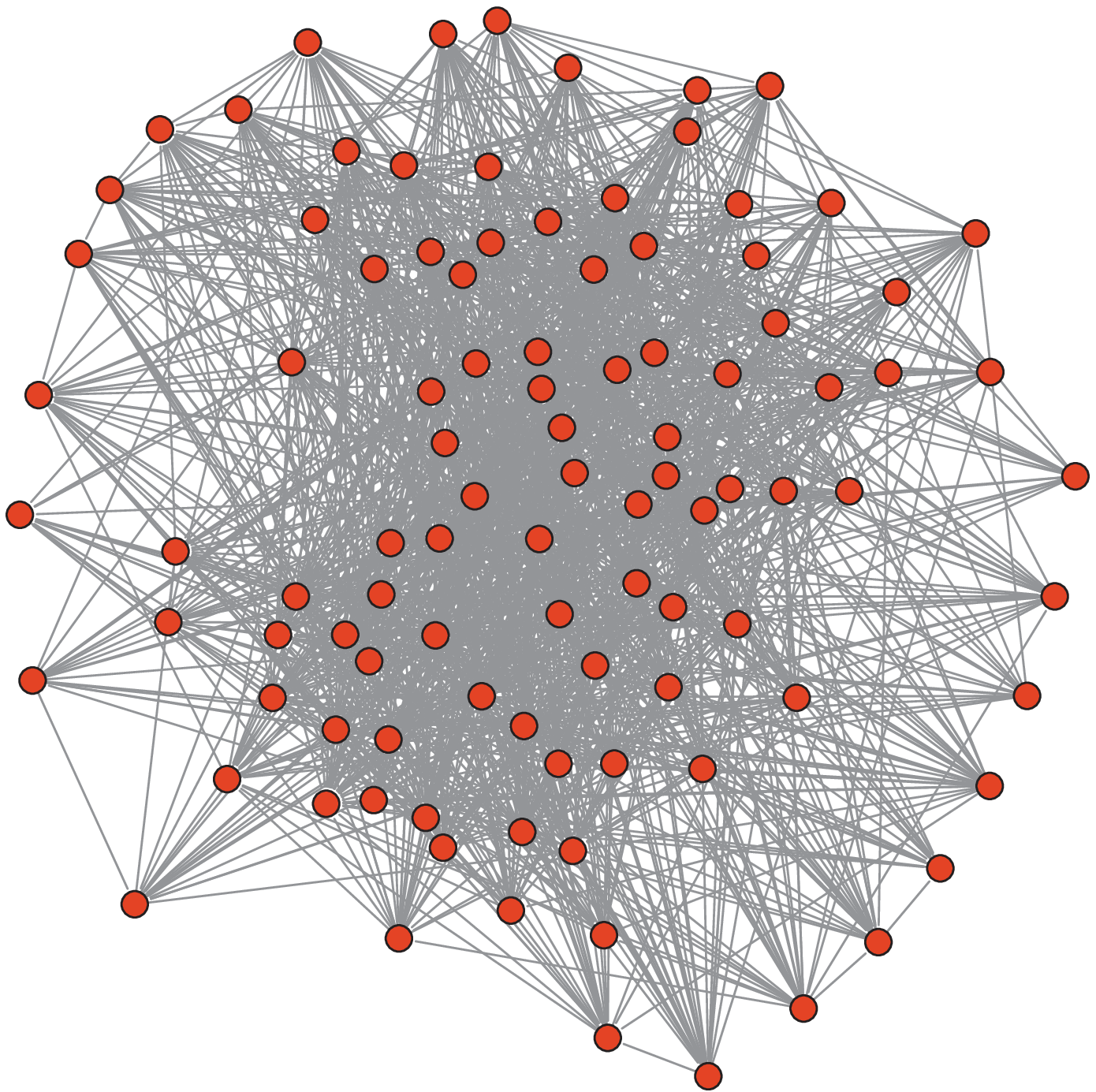}  &
  \includegraphics[scale=0.36]{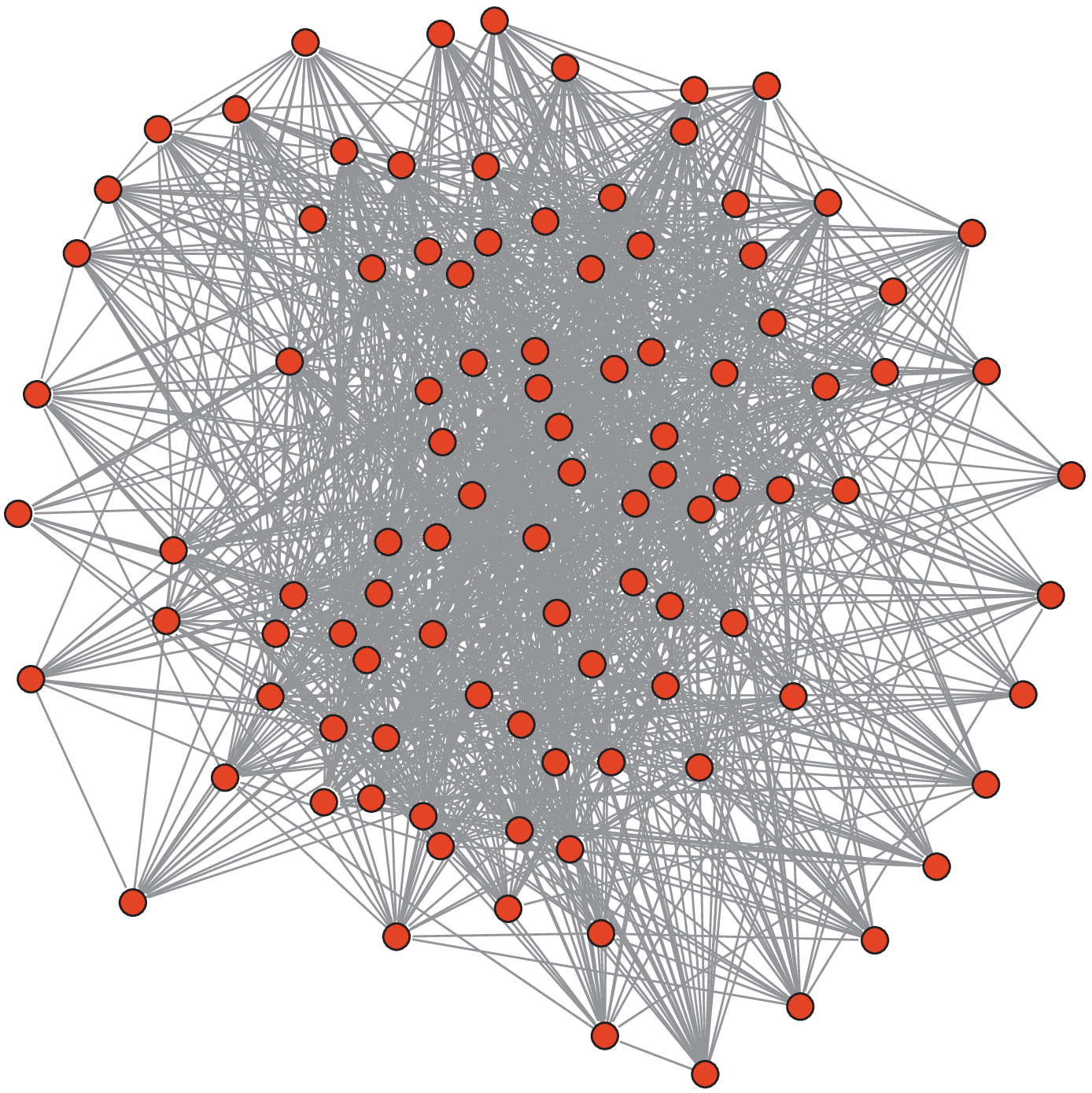} &
  \includegraphics[scale=0.36]{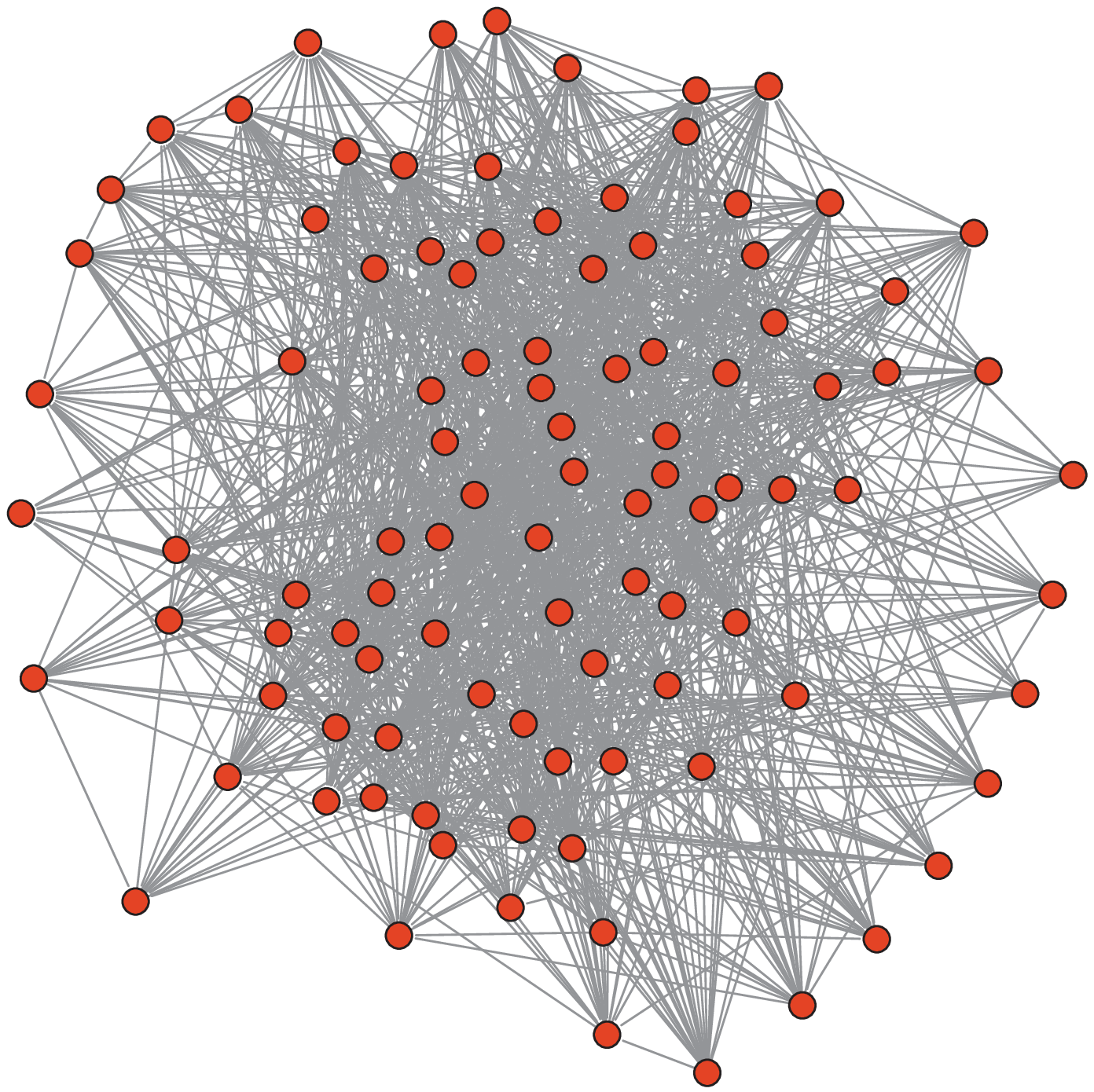}   \\
  (d) $K$-CV  graph & (e) BIC graph & (f) AIC graph
   \end{tabular}
  \end{small}
   \caption{\small Comparison of different methods on the data from the neighborhood graphs ($n=400, p=100$).}
   \label{fig:neighborhood}
\end{figure}

\begin{figure}[ht!]
\centering
  \begin{small}
  \begin{tabular}{ccc}
  \includegraphics[scale=0.36]{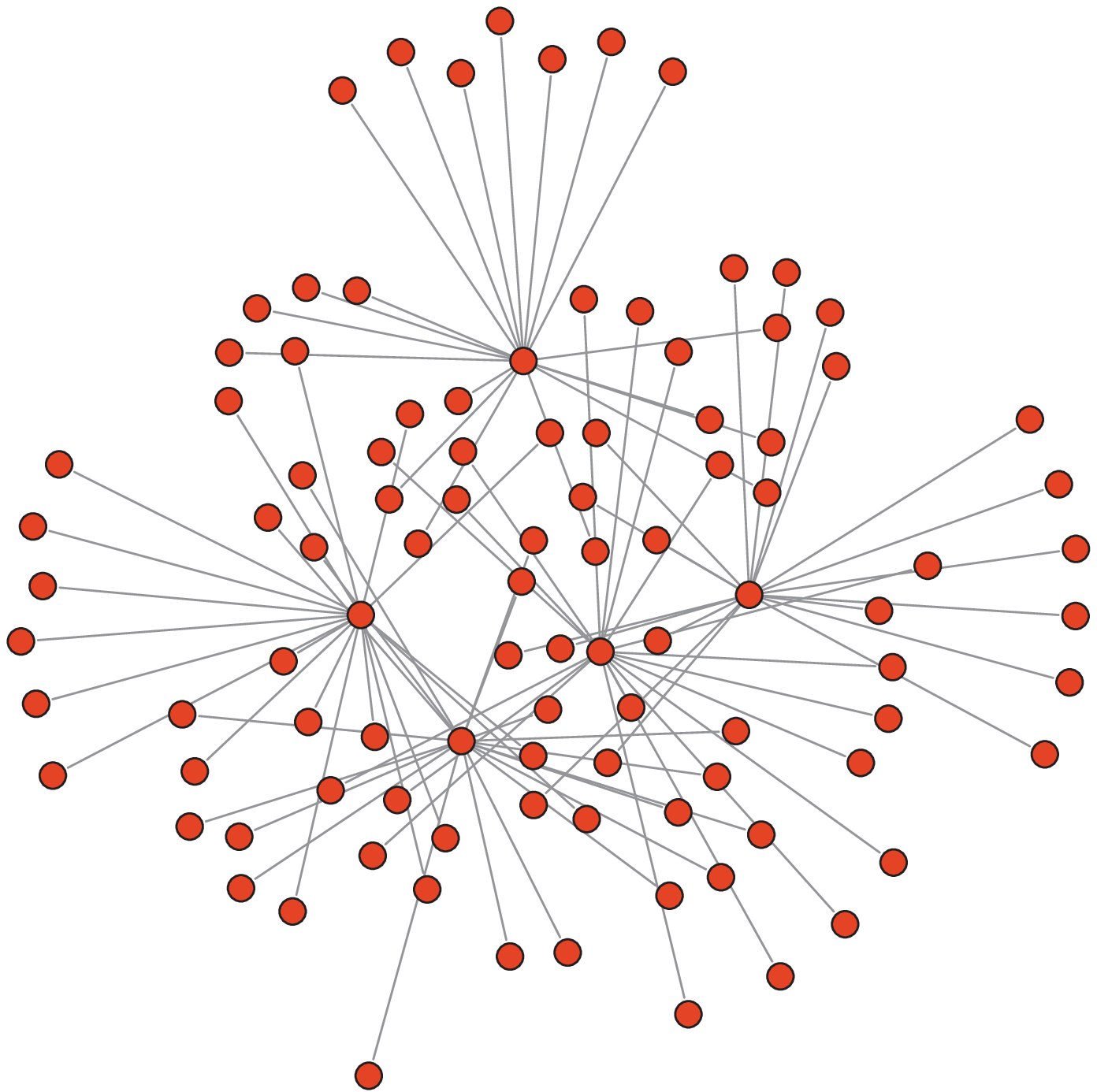} &
  \includegraphics[scale=0.36]{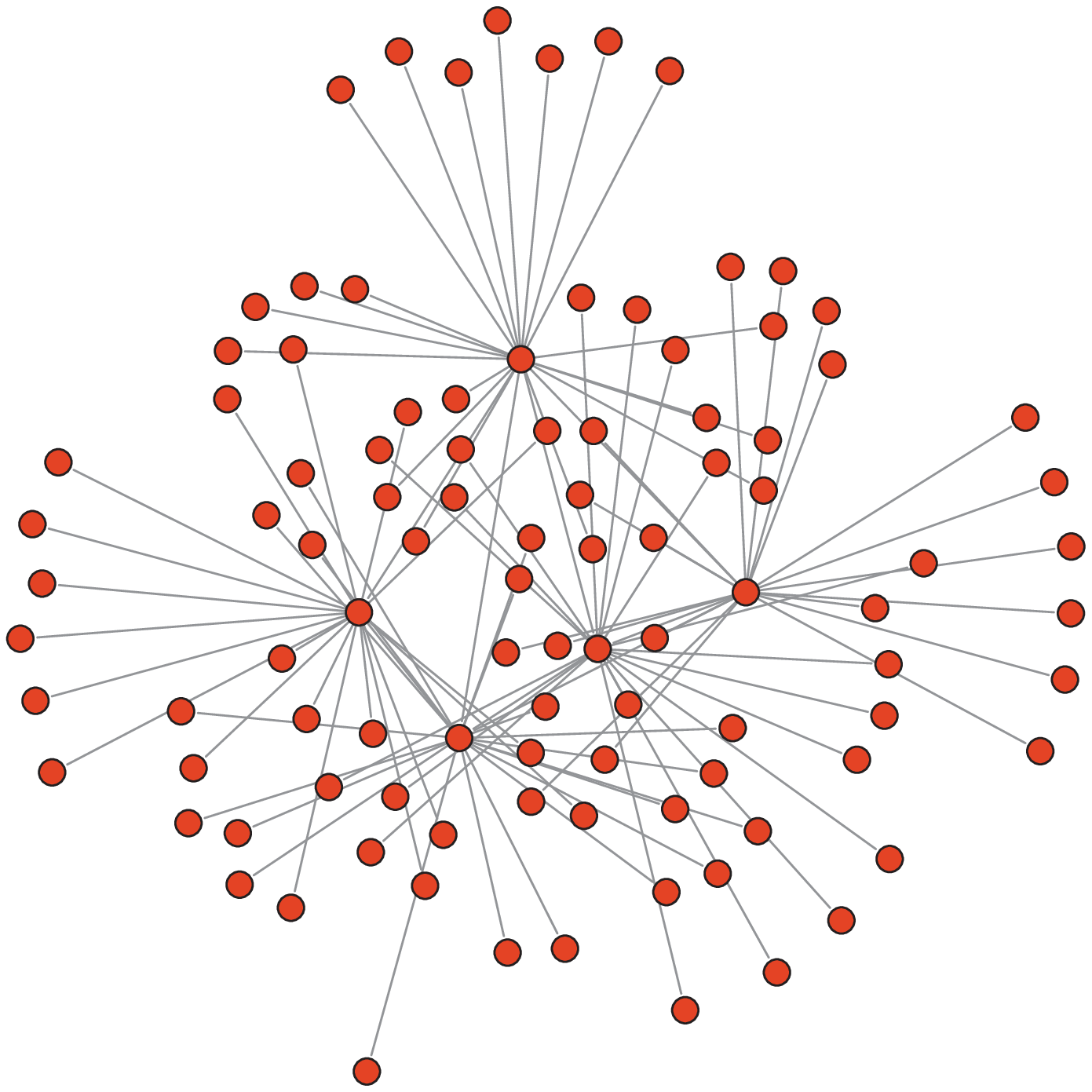} &
  \includegraphics[scale=0.36]{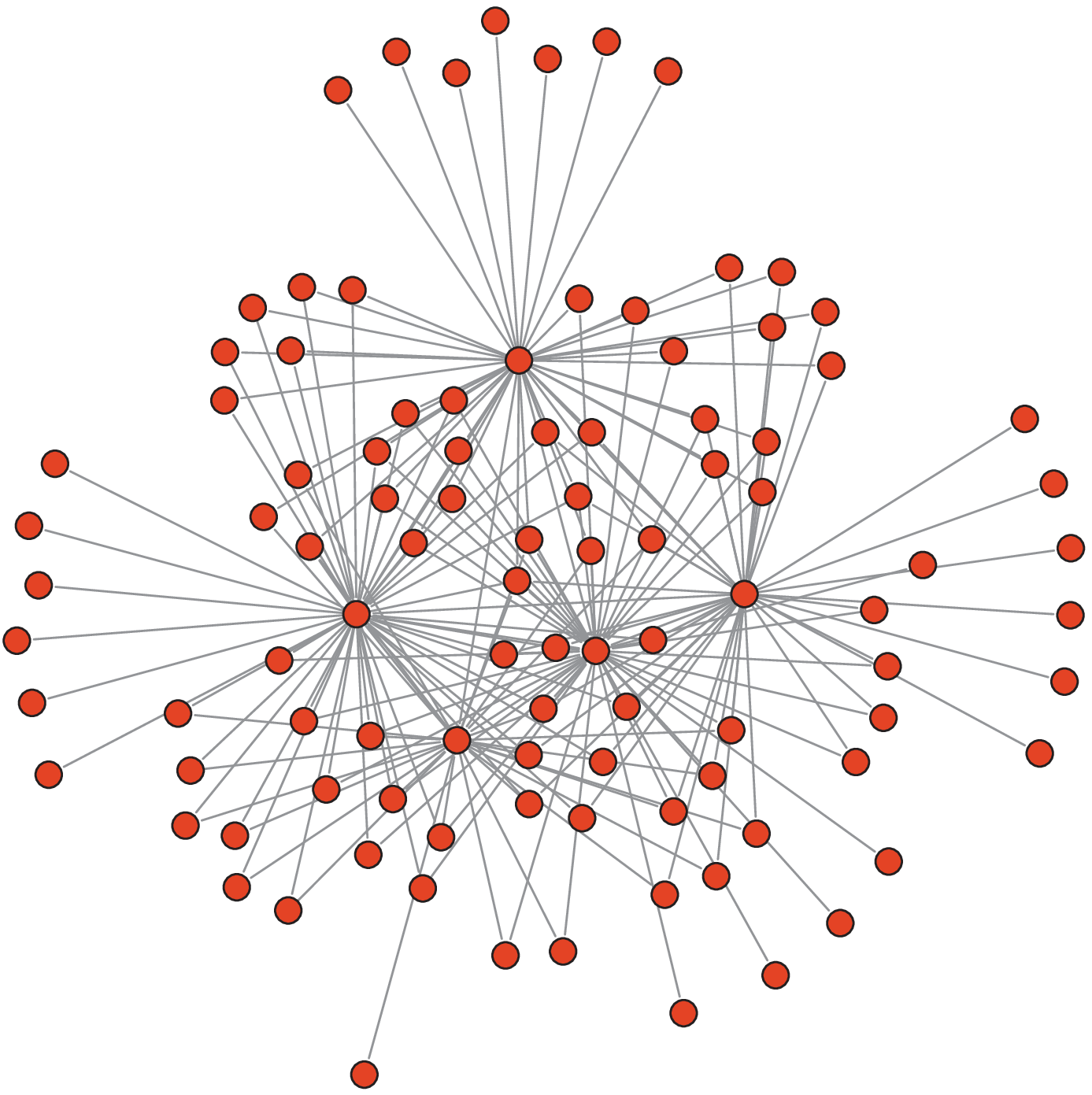}\\
     (a) True graph & (b) Oracle graph & (c) StARS graph\\
  \includegraphics[scale=0.36]{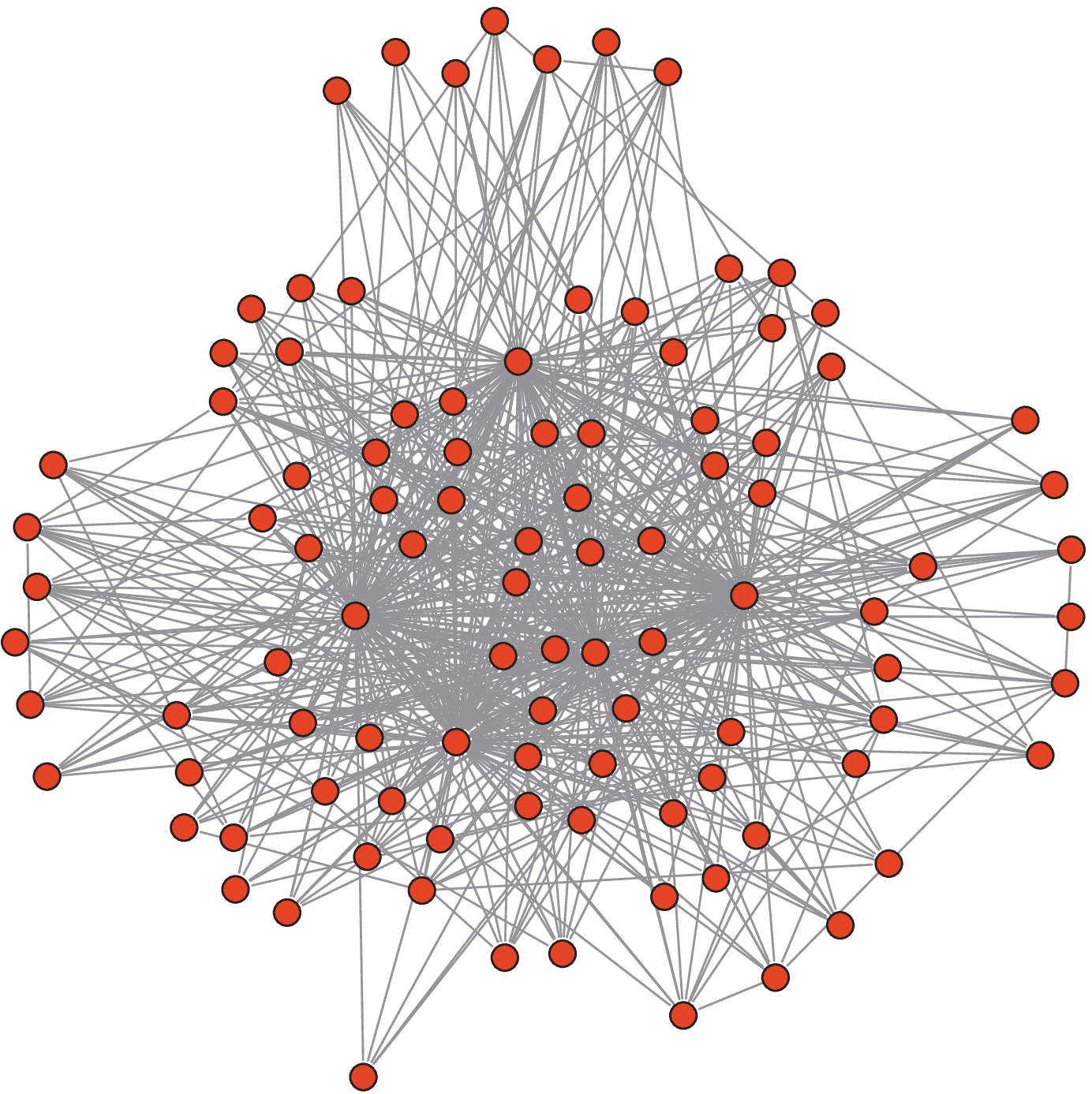}  &
  \includegraphics[scale=0.36]{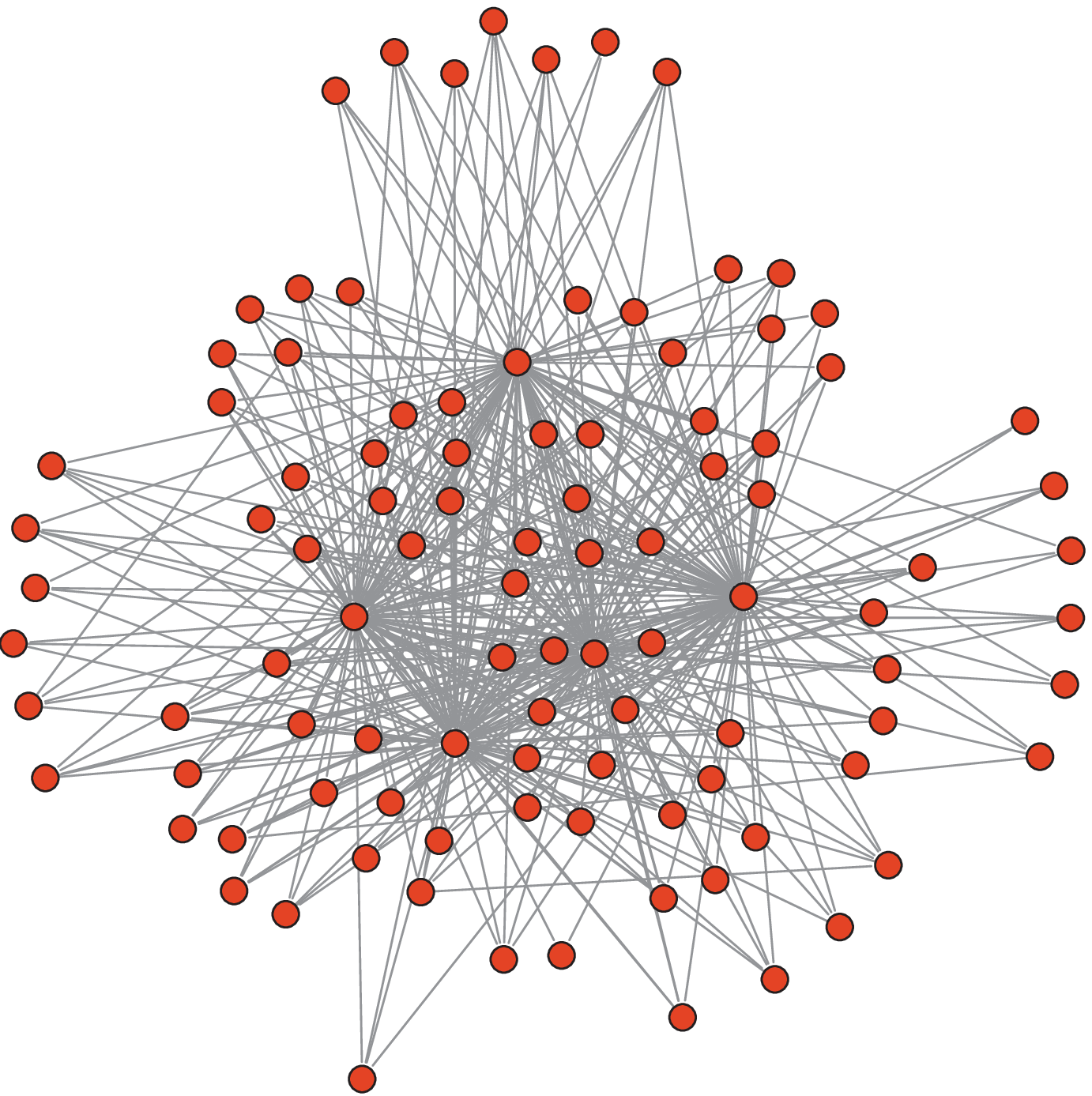} &
  \includegraphics[scale=0.36]{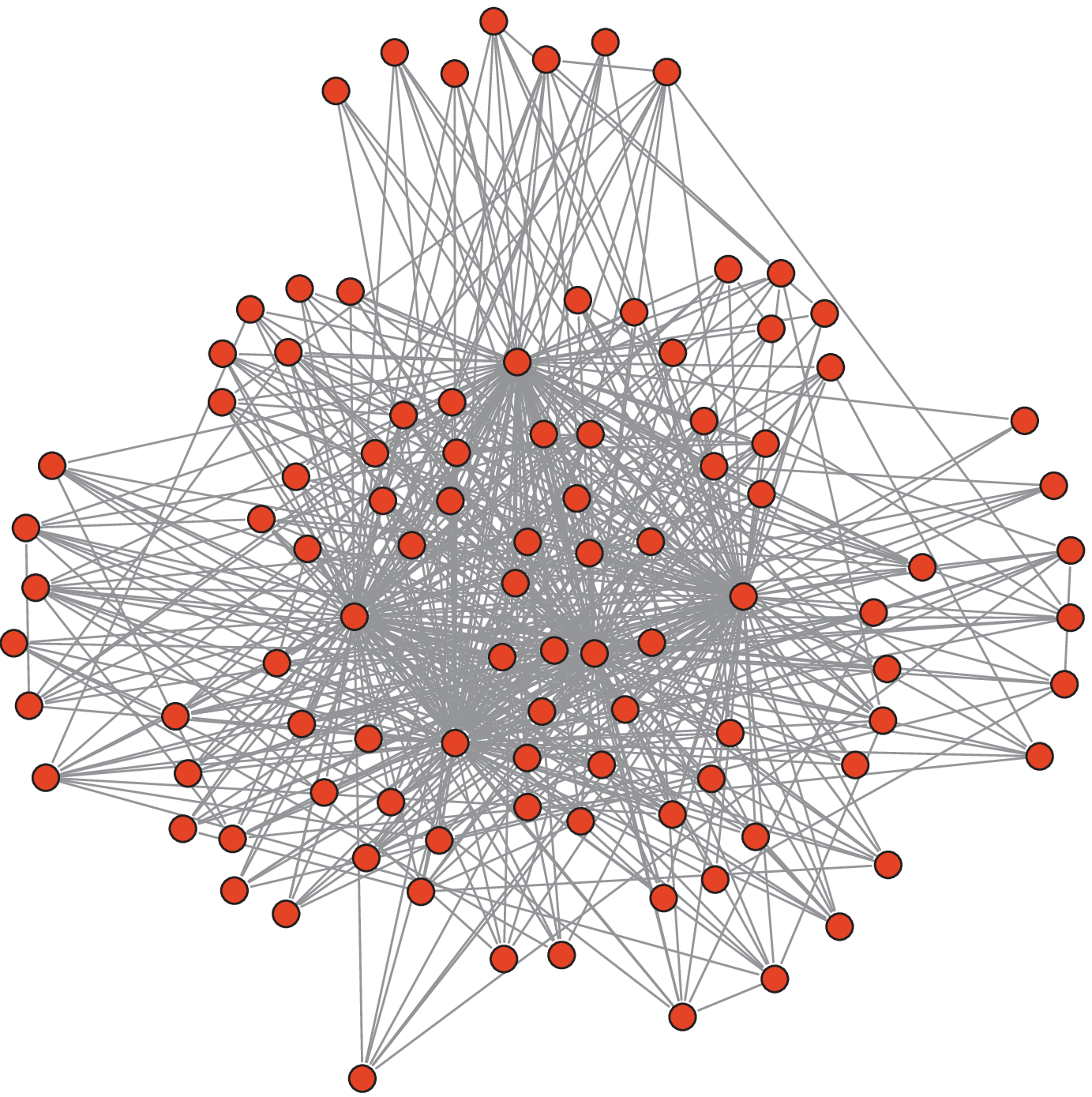}   \\
  (d) $K$-CV  graph & (e) BIC graph & (f) AIC graph
   \end{tabular}
  \end{small}
   \caption{\small Comparison of different methods on the data from the hub graphs ($n=400, p=100$).}
   \label{fig:hub}
\end{figure}

%
%
%

\subsection{Microarray Data}

We apply StARS to a dataset based on Affymetrix GeneChip microarrays
for the gene expression levels from immortalized B cells of human
subjects.  The sample size is $n=294$.  The expression levels for each
array are pre-processed by log-transformation and standardization as
in \citep{Nayak09}.  Using a previously estimated sub-pathway subset containing 324 genes \citep{Liu:10a}, we study
the estimated graphs obtained from each method under
investigation. The StARS and BIC graphs are
provided in Figure \ref{fig:microarray}. We see that the StARS graph
is remarkably simple and informative,
exhibiting some cliques and hub genes.  In contrast, the BIC graph is
very dense and possible useful association information is buried in the large
number of estimated edges. The selected graphs using AIC and $K$-CV are even more dense than the BIC graph and is omitted here.
 A full treatment of the biological implication of these two graphs validated by enrichment analysis will be left as a future study.

\begin{figure}[ht!]
\centering
  \begin{small}
  \begin{tabular}{cc}
  \includegraphics[scale=0.5]{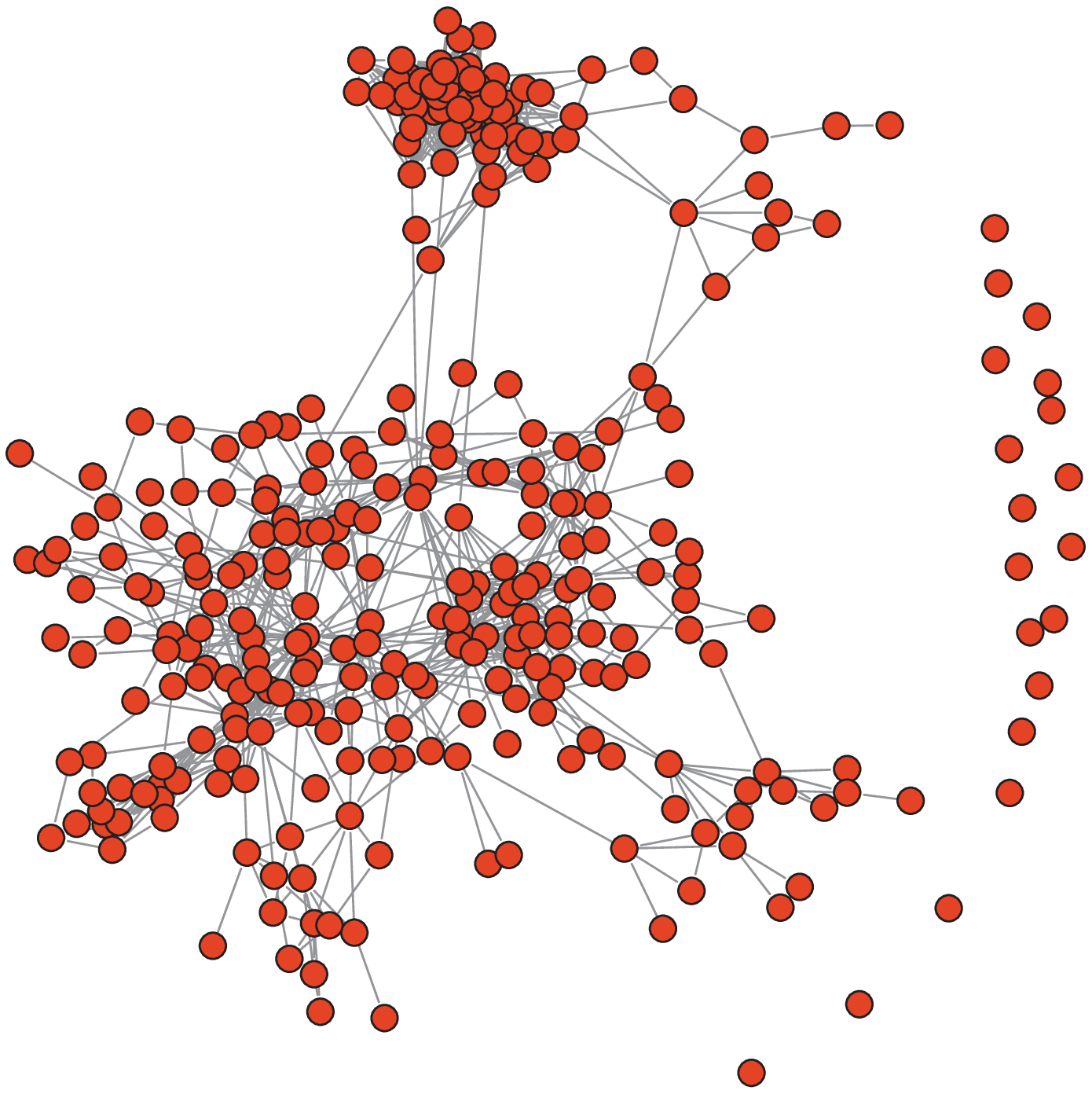} &
\hspace{0.5cm}  \includegraphics[scale=0.5]{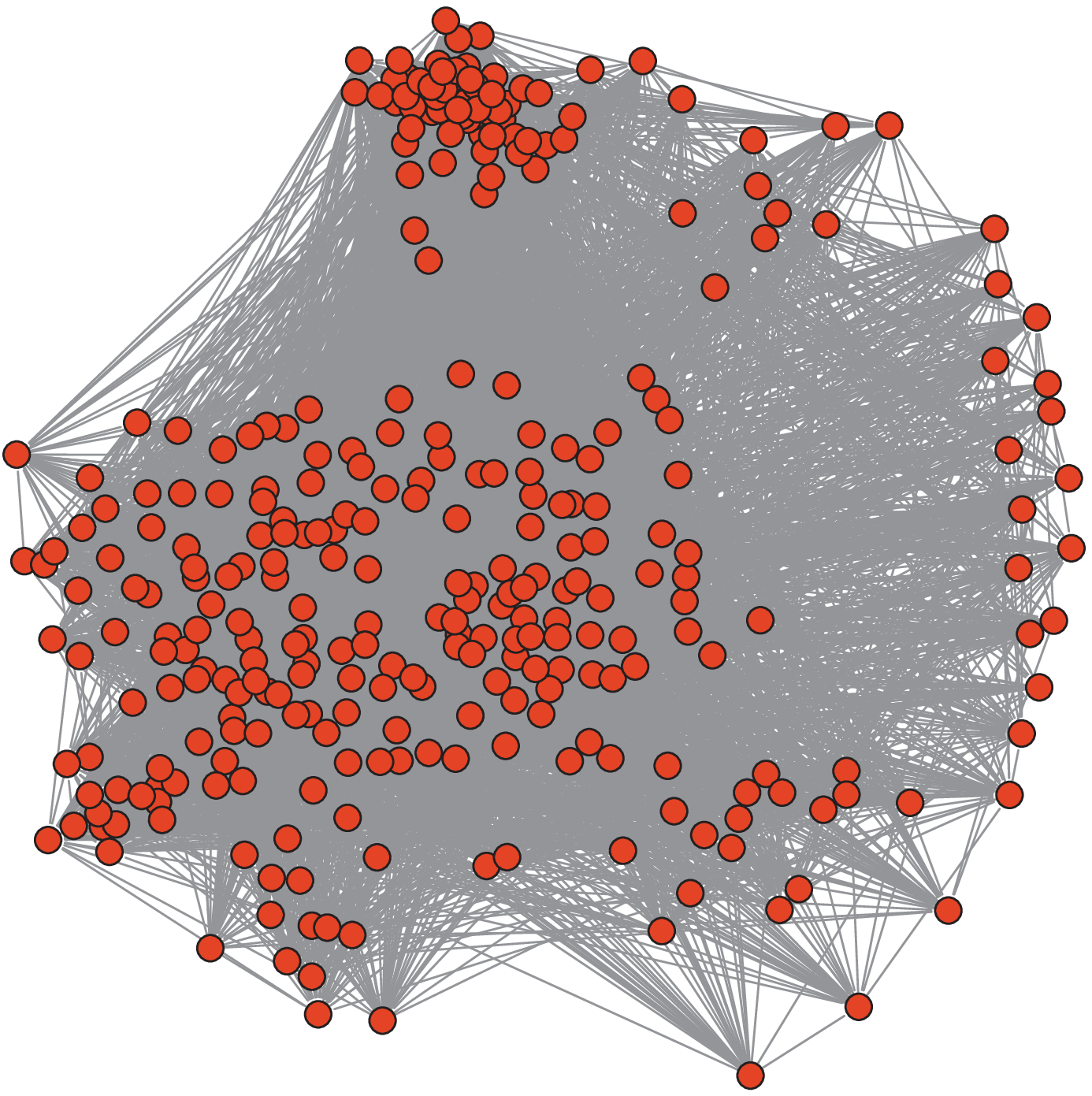} \\
     (a) StARS graph & (b) BIC graph\\
   \end{tabular}
  \end{small}
   \caption{\small Microarray data example. The StARS graph is more informative graph than the BIC graph.}
   \label{fig:microarray}
\end{figure}

\section{Conclusions}

The problem of estimating structure in high dimensions is
very challenging. Casting the problem in the
context of a regularized optimization has led to some success, but
the choice of the regularization parameter is critical.  We present a
new method, StARS, for choosing this parameter in high dimensional inference
for undirected graphs. Like Meinshausen and B\"uhlmann's
 stability selection approach \citep{stability:10}, our method makes
use of subsampling, but it differs substantially from their approach
in both implementation and goals. For graphical models, we choose the regularization
parameter directly based on the edge stability. Under mild conditions, StARS is partially
sparsistent.  However, even without these conditions, StARS has a
simple interpretation: we use  the least
amount of regularization that simultaneously makes a graph sparse and
replicable under random sampling.


Empirically, we show that StARS works significantly better than existing
techniques on both synthetic and microarray datasets.  Although we
focus here on graphical models, our new method is generally applicable
to many problems that involve estimating structure, including regression, classification, density estimation, clustering, and dimensionality reduction.

\bibliography{local}

\end{document}